\documentclass[11pt]{article}

\usepackage[utf8]{inputenc}
\usepackage[T1]{fontenc}
\usepackage{amsmath, amssymb, amsthm}
\usepackage{booktabs}
\usepackage{multirow}
\usepackage{array}
\usepackage{xcolor}
\usepackage{url}
\usepackage{authblk}
\usepackage{geometry}
\usepackage{float}
\usepackage{caption}
\usepackage{subcaption}
\usepackage{enumitem}
\usepackage{bm}
\usepackage{algorithm}
\usepackage{algpseudocode}

\newtheorem{theorem}{Theorem}[section]

\newtheorem{lemma}[theorem]{Lemma}

\theoremstyle{definition}

\newtheorem{remark}[theorem]{Remark}

\usepackage{pgfplots}
\pgfplotsset{compat=1.18}
\usepackage{tikz}
\usetikzlibrary{arrows.meta, positioning, fit, backgrounds, calc,
                shapes.geometric, decorations.pathreplacing}

\usepackage[hidelinks]{hyperref}
\hypersetup{colorlinks=true, linkcolor=blue!70!black,
            citecolor=blue!70!black, urlcolor=blue!70!black}

\newcommand{\bind}{\circledast}
\newcommand{\unbind}{\oslash}
\newcommand{\Rb}{\mathbb{R}}
\newcommand{\Cb}{\mathbb{C}}
\newcommand{\ev}[1]{\mathbf{e}_{#1}}
\newcommand{\rv}[1]{\boldsymbol{\rho}_{#1}}
\newcommand{\mem}{\mathbf{M}}
\DeclareMathOperator*{\argmaxop}{arg\,max}

\DeclareMathOperator{\cleanupop}{clean}
\DeclareMathOperator{\Real}{Re}

\title{\textbf{Holographic Memory for Zero-Shot Compositional\\[2pt]
Reasoning in Knowledge Graphs:\\[4pt]
A Mechanistic Study of Where and Why It Fails}}

\author[1]{Randhir Kumar}
\affil[1]{\small Independent Researcher \\ \texttt{randhir2709vns@gmail.com}}

\date{June 2026}

\begin{document}
\maketitle

\begin{abstract}
Knowledge graph embedding (KGE) models predict single-hop links well but have
no mechanism for \emph{zero-shot compositional} queries: multi-hop questions
whose relation chains never appeared during training. Holographic Reduced
Representations (HRR), which bind and unbind symbols via circular
convolution, are a theoretically attractive candidate, since binding is
approximately invertible and associative. We test whether this promise holds.

We study two holographic memory variants, real-valued HRR and phase-only
Fourier HRR (FHRR), each with a modern Hopfield cleanup, on FB15k-237 over
five seeds. Four findings follow. First, both are competitive single-hop
retrievers (filtered MRR $0.358 \pm 0.002$ for HRR, $0.350 \pm 0.021$ for
FHRR). Second, neither composes zero-shot: accuracy stays at chance across
all cleanup temperatures. Third, the main contribution, we localise the
failure mechanistically. A hop-1 probe shows the memory recovers the correct
intermediate entity with high fidelity (MRR $0.896 \pm 0.002$ for HRR), yet
composition still fails even with a verified-correct intermediate. A second
probe shows why: posing the \emph{ground-truth} second-hop fact as a
standalone atomic query, bypassing composition entirely, already recovers it
at only $0.26$ to $0.48\times$ average atomic accuracy, uniformly across
relation fan-out. The bottleneck is not the bind-unbind algebra or the
cleanup; it is that facts compositional chains pass through are intrinsically
harder for the superposed memory to retrieve, a capacity and interference
effect present already at a single hop. Fourth, we prove
(Lemma~\ref{lem:softmax-noncommute}) that FHRR's softmax cleanup is not
phase-equivariant, compounding the primary failure on the minority of chains
where hop-1 itself errs. Fixing zero-shot composition requires improving
retrieval capacity under superposition, not just redesigning the cleanup.
\end{abstract}


% ============================================================
\section{Introduction}
\label{sec:intro}
% ============================================================

Knowledge graphs (KGs) organise world knowledge as typed triples $(h, r, t)$,
e.g.\ \emph{(Marie\_Curie, nationality, Poland)}, and KGE methods embed entities
and relations into vector spaces so that plausible triples score high
\cite{bordes2013transe,yang2014distmult,trouillon2016complex,sun2019rotate}.
These methods work well for single-hop prediction. Many real queries are not
single-hop, though. "Which administrative regions contain the birthplace of
a given person's nationality?" requires chaining \textsc{nationality} with
\textsc{contains}. If that exact chain never appeared in training, a standard
KGE model cannot answer it: it has no mechanism to try. That is the
\emph{zero-shot compositional reasoning} setting we study here.

\paragraph{Why holographic memory?}
Holographic Reduced Representations
(HRR)~\cite{plate2003holographic} and the broader family of Vector Symbolic
Architectures (VSA)~\cite{kanerva2009hyperdimensional,schlegel2022comparison}
define an algebra over fixed-width distributed vectors in which arbitrary symbol
structures can be \emph{bound} into a single vector and later \emph{unbound}.
The key operation, circular convolution, runs in $O(D \log D)$, is approximately
invertible, and is associative~\cite{plate2003holographic}. A holographic memory
that superposes facts as bound triples and answers queries by composing unbind
operations is a natural zero-shot compositional reasoner, in principle. Modern
Hopfield networks~\cite{ramsauer2021hopfield} give a fully differentiable
associative cleanup step. The question is whether the pipeline actually
composes in practice, and if not, which part breaks and why.

\paragraph{Scope.}
We do not claim to outperform supervised compositional methods such as
Query2Box~\cite{ren2020query2box} or CQD~\cite{arakelyan2021complex}. Those
methods receive explicit path-level supervision; we withhold it. The goal is to
understand what holographic binding alone can do for zero-shot composition,
without any relation-chain supervision. We study FB15k-237~\cite{toutanova2015observed}
under a leakage-controlled two-hop protocol and verify all claims over five
independent seeds.

\paragraph{Contributions.}
\begin{enumerate}[leftmargin=1.4em, itemsep=3pt]
\item \textbf{Both variants are competitive atomic retrievers.}
Real HRR reaches filtered MRR $0.358 \pm 0.002$ and FHRR $0.350 \pm 0.021$,
in the range of TransE and DistMult on the same benchmark. A Hopfield ablation
shows cleanup accounts for roughly half the performance
(Section~\ref{sec:results-atomic}).

\item \textbf{Both variants fail at zero-shot two-hop composition.}
Per-seed binomial tests against chance are non-significant in most seeds for
both models; accuracy is flat across all tested cleanup temperatures. Training
gives no compositional advantage over an untrained control
(Section~\ref{sec:results-zeroshot}).

\item \textbf{We localise the failure to retrieval capacity, not the cleanup algebra.}
A hop-1 probe shows the intermediate entity is recovered with high fidelity
(mid-entity MRR $\approx 0.85$ to $0.90$), yet composition stays at chance even
when the intermediate is verified correct. A second probe shows that standalone
atomic accuracy on the \emph{ground-truth} second-hop fact is itself degraded
to $0.26$ to $0.48\times$ the model's average atomic accuracy, independent of
relation fan-out. The facts compositional chains rely on are simply harder to
retrieve from the superposed memory, and this is true before composition or
cleanup enters the picture (Section~\ref{sec:results-mechanism}).

\item \textbf{We prove a compounding secondary failure in FHRR.}
Lemma~\ref{lem:softmax-noncommute} shows the softmax Hopfield cleanup does not
commute with phase-additive binding. This compounds the primary failure on the
subset of chains where hop-1 retrieval is imperfect
(Section~\ref{sec:results-mechanism}).
\end{enumerate}

% ============================================================
\section{Related Work}
\label{sec:related}
% ============================================================

\paragraph{Knowledge graph embeddings.}
TransE~\cite{bordes2013transe} models a relation as a translation in entity
space, with score $\|h + r - t\|$. DistMult~\cite{yang2014distmult} uses a
bilinear diagonal score; ComplEx~\cite{trouillon2016complex} extends it to
complex embeddings to handle asymmetric relations; RotatE~\cite{sun2019rotate}
treats relations as element-wise rotations in complex space. All four are strong
single-hop predictors but offer no native mechanism for composing unseen
relation chains at test time.

\paragraph{Compositional and multi-hop reasoning.}
Several lines of work address multi-hop reasoning with explicit supervision.
Guu et al.~\cite{guu2015traversing} compose relation embeddings along observed
paths; NeuralLP~\cite{yang2017neurallp} and RNNLogic~\cite{qu2021rnnlogic}
learn soft logical rules. Query2Box~\cite{ren2020query2box} and
BetaE~\cite{ren2020beta} embed existential first-order queries as geometric
objects, training on path-structured supervision.
CQD~\cite{arakelyan2021complex} decomposes complex queries into atomic link
predictions at inference time. Our setting is strictly harder: zero-shot
composition with no path-level supervision and no explicit intermediate-entity
representations.

\paragraph{Holographic and vector-symbolic memory.}
HRR~\cite{plate2003holographic} introduced circular convolution as a binding
operator, building on tensor-product variable binding~\cite{smolensky1990tensor}.
Hyperdimensional computing~\cite{kanerva2009hyperdimensional} and the
comprehensive VSA survey~\cite{schlegel2022comparison} cover related algebras.
The FHRR variant~\cite{plate2003holographic,schlegel2022comparison} encodes
symbols as unit-modulus phasors, with binding as element-wise complex
multiplication. Standard noise analysis shows recovery error growing as
$O(\sqrt{K/D})$ in the number of superposed facts $K$
\cite{plate2003holographic}. Our mechanistic probes
(Section~\ref{sec:results-mechanism}) measure this capacity effect directly for
the facts compositional reasoning depends on, rather than for the memory in
aggregate. Resonator networks~\cite{frady2020resonator} propose an iterative
cleanup that is phase-equivariant by construction and may partially address the
secondary failure mode in Lemma~\ref{lem:softmax-noncommute}. Our results
indicate the primary bottleneck sits upstream of cleanup, in retrieval capacity
itself.

\paragraph{Modern Hopfield networks.}
Ramsauer et al.~\cite{ramsauer2021hopfield} showed that the update rule of
continuous modern Hopfield networks coincides with scaled dot-product attention,
with storage capacity exponential in $D$. We use this as a differentiable
cleanup that maps a noisy unbound estimate back to the entity codebook. The
capacity results concern single-pattern recovery from a nearby query. They do
not directly bound retrieval accuracy for a fact buried in a superposition of
$K = 272{,}115$ others at varying frequency, which is what our probes measure.

% ============================================================
\section{Problem Formulation}
\label{sec:problem}
% ============================================================

Let $\mathcal{E}$ and $\mathcal{R}$ denote finite sets of entities and
relations respectively; a knowledge graph is a set of triples
$\mathcal{T} \subseteq \mathcal{E} \times \mathcal{R} \times \mathcal{E}$,
partitioned into disjoint training, validation, and test sets
$\mathcal{T}_{\mathrm{tr}}, \mathcal{T}_{\mathrm{va}}, \mathcal{T}_{\mathrm{te}}$.

\paragraph{Single-hop (atomic) task.}
Given a query pair $(h, r) \in \mathcal{E} \times \mathcal{R}$, predict $t$
such that $(h, r, t) \in \mathcal{T}_{\mathrm{te}}$. Performance uses filtered
Mean Reciprocal Rank (MRR) and Hits@$k$ under the standard filtered
protocol~\cite{bordes2013transe}.

\paragraph{Two-hop compositional task.}
For a relation chain $(r_1, r_2) \in \mathcal{R}^2$, the \emph{compositional
answer set} for head entity $h$ is
\begin{equation}
A_{h,r_1,r_2}
  = \bigl\{\, t \in \mathcal{E} : \exists\, m \in \mathcal{E},\;
    (h, r_1, m) \in \mathcal{T} \;\wedge\; (m, r_2, t) \in \mathcal{T} \,\bigr\}.
\label{eq:composition-set}
\end{equation}
The model must predict an element of $A_{h,r_1,r_2}$ given only $(h, r_1, r_2)$,
with no training on the composite $(r_1 \circ r_2)$ and no access to any
intermediate entity $m$.

\paragraph{Zero-shot protocol and leakage control.}
We extract the ten highest-frequency two-hop relation chains in FB15k-237 with
chain training support $\geq 100$ (the number of distinct mid-entities $m$
appearing in both atomic triples of the chain in $\mathcal{T}_{\mathrm{tr}}$).
Individual relations $r_1, r_2$ appear in training; only their composition is
withheld. To remove learnable shortcuts, we discard any test pair $(h, t)$
where $t$ is directly reachable from $h$ via a single training relation,
removing $98$ pairs ($0.14\%$) and leaving $69{,}855$ genuinely zero-shot
pairs (Algorithm~\ref{alg:split}). Under a uniform single-answer ranking
assumption, chance accuracy is $1/|\mathcal{E}| \approx 6.77 \times 10^{-5}$;
this is our reference null throughout. Since several compositional answer sets
admit more than one valid tail, the true chance rate for some queries exceeds
$1/|\mathcal{E}|$, which makes the reported $p$-values conservative in the
direction of rejecting $H_0$ less often. The selected chains span 19 distinct
relation types and mid-entity fan-out from $2.4$ to $18.7$.

\begin{algorithm}[t]
\caption{Leakage-controlled zero-shot evaluation set construction}
\label{alg:split}
\begin{algorithmic}[1]
\Require Training graph $G_{\mathrm{tr}}$, test graph $G_{\mathrm{te}}$,
         number of chains $n$, minimum support $s_{\min}$
\For{each ordered pair $(r_1, r_2) \in \mathcal{R}^2$}
  \State Compute chain support: 
         $\text{sup}(r_1,r_2) \gets |\{m : \exists h,t, (h,r_1,m) \in G_{\mathrm{tr}}, 
         (m,r_2,t) \in G_{\mathrm{tr}}\}|$
\EndFor
\State Select the $n$ chains with most test pairs subject to 
       $\text{sup}(r_1,r_2) \geq s_{\min}$
\State $\mathcal{S} \gets \emptyset$
\For{each selected chain $(r_1,r_2)$ and test pair $(h,t)$ reachable by it}
  \If{$t \in \mathcal{N}_{\mathrm{tr}}(h)$} 
    \quad \Comment{$t$ is a direct training neighbour of $h$}
    \State \textbf{discard} $(h, t)$
  \Else
    \State $\mathcal{S} \gets \mathcal{S} \cup \{(h, r_1, r_2, t)\}$
  \EndIf
\EndFor
\Ensure $\mathcal{S}$ \Comment{$69{,}855$ zero-shot evaluation quadruples}
\end{algorithmic}
\end{algorithm}

% ============================================================
\section{Method}
\label{sec:method}
% ============================================================

\subsection{Holographic Binding Algebra}
\label{sec:method-algebra}

\paragraph{DFT convention.}
We use the unitary discrete Fourier transform:
$\hat{a}_k = D^{-1/2}\sum_{j=0}^{D-1} a_j\, e^{-2\pi i jk/D}$,
so $\|\hat{\mathbf{a}}\|_2 = \|\mathbf{a}\|_2$.

\paragraph{Real HRR.}
For $\mathbf{a}, \mathbf{b} \in \Rb^D$, circular convolution is
\begin{align}
  (\mathbf{a} \bind \mathbf{b})_k 
    &= \sum_{j=0}^{D-1} a_j\, b_{(k-j) \bmod D},
  \label{eq:cconv}\\
  \widehat{\mathbf{a} \bind \mathbf{b}} 
    &= \hat{\mathbf{a}} \odot \hat{\mathbf{b}},
  \label{eq:cconv-freq}
\end{align}
where $\odot$ is element-wise multiplication. Binding is commutative, and for
embeddings drawn i.i.d.\ from $\mathcal{N}(0, D^{-1}\mathbf{I})$, the bound
vector has approximately the same norm as the operands. Unbinding uses the
approximate inverse:
\begin{equation}
  \mathbf{a} \approx (\mathbf{a} \bind \mathbf{b}) \unbind \mathbf{b}
    \;=\; \mathcal{F}^{-1}\!\bigl(
      \widehat{\mathbf{a} \bind \mathbf{b}} \odot \overline{\hat{\mathbf{b}}}
    \bigr).
  \label{eq:unbind}
\end{equation}
Approximation error grows with the number of superposed facts $K$, at rate
$O(\sqrt{K/D})$ for random embeddings. With $K = 272{,}115$ training triples
in $D = 1024$ real dimensions, $K/D \approx 266$, a substantial noise floor.
Section~\ref{sec:results-mechanism} measures its differential effect on the
specific facts compositional chains depend on.

\paragraph{Complex FHRR.}
FHRR encodes each symbol as a unit-modulus phasor $\mathbf{z} = e^{i\boldsymbol{\phi}}$
with $\boldsymbol{\phi} \in [-\pi, \pi]^D$. Binding is element-wise complex
multiplication (phase addition); unbinding multiplies by the conjugate:
\begin{equation}
  (\mathbf{z}_A \bind \mathbf{z}_B)_k = z_{A,k} \cdot z_{B,k} = e^{i(\phi_{A,k}+\phi_{B,k})},
  \qquad
  (\mathbf{z}_A \bind \mathbf{z}_B) \unbind \mathbf{z}_B = \mathbf{z}_A.
  \label{eq:fhrr-bind}
\end{equation}
Recovery is exact for a single stored fact and noisy for a superposition of $K$
facts, where phase interference accumulates across components.
Figure~\ref{fig:phasor-binding} visualises binding and unbinding on the unit circle.

\begin{figure}[H]
\centering
\resizebox{0.95\textwidth}{!}{%
\begin{tikzpicture}[
  scale=1.0,
  circ/.style={circle, draw=black, thick, minimum size=3.2cm},
  phasor/.style={->, thick, >=Stealth, #1},
]
\coordinate (origL) at (0,0);
\draw[circ] (origL) circle (1.5cm);
\draw[gray!40] (-1.5,0) -- (1.5,0);
\draw[gray!40] (0,-1.5) -- (0,1.5);
\draw[phasor=blue!70] (origL) -- (30:1.5);
\node[blue!70, anchor=south] at (30:1.6) {$\ev{A}$};
\draw[blue!50, dashed] (30:1.5) arc (30:0:0.5);
\node[blue!50, right] at (0.6,0.2) {$\theta_A$};
\draw[phasor=orange!70] (origL) -- (50:1.5);
\node[orange!70, anchor=south west] at (50:1.6) {$\rv{R}$};
\draw[orange!50, dashed] (50:1.5) arc (50:0:0.9);
\node[orange!50, right] at (1.0,0.35) {$\theta_R$};
\draw[phasor=green!60!black, line width=1.2pt] (origL) -- (80:1.5);
\node[green!60!black, anchor=south] at (80:1.6) {$\ev{A} \bind \rv{R}$};
\draw[green!60!black, dashed] (80:1.5) arc (80:0:1.3);
\node[green!60!black, right] at (1.4,0.55) {$\theta_A{+}\theta_R$};
\node[draw, fill=yellow!15, rounded corners, below=1.5cm of origL,
      minimum width=5cm]
      {$\displaystyle \text{Bind: } (e^{i\theta_A}) \odot (e^{i\theta_R}) = e^{i(\theta_A + \theta_R)}$};
\begin{scope}[xshift=6cm]
  \coordinate (origR) at (0,0);
  \draw[circ] (origR) circle (1.5cm);
  \draw[gray!40] (-1.5,0) -- (1.5,0);
  \draw[gray!40] (0,-1.5) -- (0,1.5);
  \draw[phasor=green!60!black, dashed] (origR) -- (80:1.5);
  \node[green!60!black, anchor=south] at (80:1.6) {$\ev{A} \bind \rv{R}$};
  \draw[phasor=orange!70, densely dotted] (origR) -- (-50:1.5);
  \node[orange!70, anchor=north] at (-50:1.6) {$\overline{\rv{R}}$};
  \draw[orange!50, dashed] (-50:1.5) arc (-50:0:0.9);
  \node[orange!50, left] at (-1.0,0.35) {$-\theta_R$};
  \draw[phasor=blue!70, line width=1.2pt] (origR) -- (30:1.5);
  \node[blue!70, anchor=south] at (30:1.6) {$\approx \ev{A}$};
\end{scope}
\node[draw, fill=yellow!15, rounded corners, below=1.5cm of origR,
      minimum width=5cm]
      {$\displaystyle \text{Unbind: } (e^{i(\theta_A+\theta_R)}) \odot (e^{-i\theta_R}) = e^{i\theta_A}$};
\draw[->, thick, gray] (2.5,-1.0) -- (3.5,-1.0) 
      node[midway, above, font=\footnotesize] {unbind with $\overline{\rv{R}}$};
\end{tikzpicture}%
}
\caption{Complex-phasor binding and unbinding in FHRR. \textbf{Left:} Binding
adds phases. \textbf{Right:} Unbinding subtracts them via conjugate multiplication,
recovering the original entity up to cross-talk noise from other superposed facts.}
\label{fig:phasor-binding}
\end{figure}
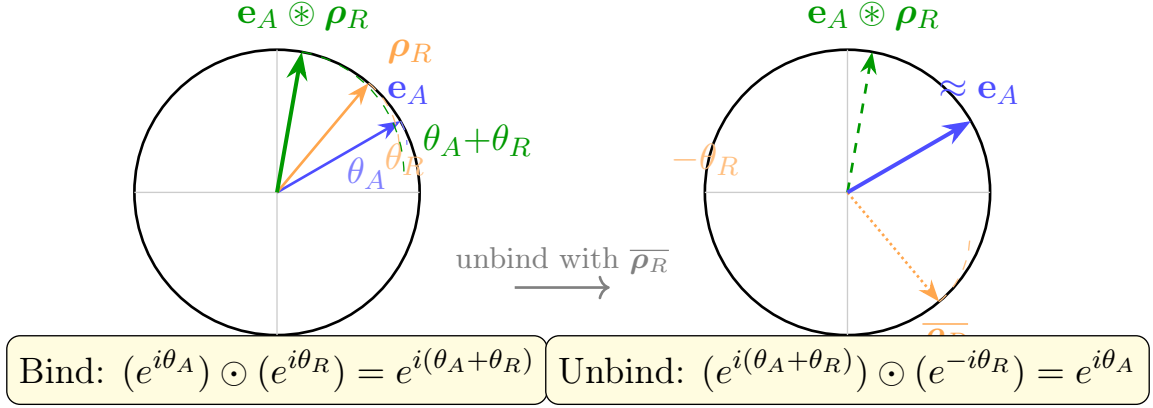

\paragraph{Memory construction.}
Both models superpose training triples into a single memory vector:
\begin{equation}
  \mem = \sum_{(h,r,t)\,\in\,\mathcal{T}_{\mathrm{tr}}} \ev{h} \bind \rv{r} \bind \ev{t},
  \label{eq:memory}
\end{equation}
where $\ev{x} \in \Rb^D$ (HRR) or $\Cb^D$ (FHRR) are learned entity embeddings
and $\rv{r}$ are learned relation embeddings. Entity and relation embeddings are
trained jointly to maximise single-hop retrieval accuracy; $\mem$ is recomputed
from the trained embeddings without further gradient updates.

\subsection{Hopfield Cleanup}
\label{sec:method-cleanup}

Unbinding from a superposed memory produces a noisy estimate
$\tilde{\mathbf{z}} \approx \ev{t}$, corrupted by interference from the other
stored facts. A modern Hopfield network~\cite{ramsauer2021hopfield} maps this
estimate back to the entity codebook $\mathbf{E} \in \Rb^{|\mathcal{E}| \times D}$:
\begin{equation}
  \cleanupop(\tilde{\mathbf{z}})
    = \boldsymbol{\alpha}(\tilde{\mathbf{z}})^\top \mathbf{E},
  \qquad
  \alpha_i(\tilde{\mathbf{z}})
    = \frac{\exp\!\bigl(\beta\,\ev{i}^\top \tilde{\mathbf{z}}\bigr)}
           {\sum_{j} \exp\!\bigl(\beta\,\ev{j}^\top \tilde{\mathbf{z}}\bigr)},
  \label{eq:cleanup}
\end{equation}
where $\beta > 0$ is an inverse temperature. A hard variant uses
$\cleanupop_{\mathrm{hard}}(\tilde{\mathbf{z}}) = \ev{\argmaxop_i\,\ev{i}^\top\tilde{\mathbf{z}}}$.

For FHRR, the output is re-projected onto the unit torus by extracting the
element-wise phase:
\begin{equation}
  \cleanupop_{\mathrm{F}}(\tilde{\mathbf{z}})
    = \exp\!\bigl(i\,\angle\bigl(\boldsymbol{\alpha}(\tilde{\mathbf{z}})^\top \mathbf{E}\bigr)\bigr).
  \label{eq:cleanup-fhrr}
\end{equation}
This guarantees $|\cleanupop_{\mathrm{F}}(\tilde{\mathbf{z}})_k|=1$ by construction,
ruling out magnitude collapse as a failure mode (Section~\ref{sec:results-mechanism}).

\subsection{Atomic and Compositional Inference}
\label{sec:method-inference}

\paragraph{Atomic prediction.}
The tail estimate for query $(h, r)$ is:
\begin{equation}
  \tilde{\mathbf{z}}_t = \mem \unbind (\ev{h} \bind \rv{r}),
  \label{eq:atomic-query}
\end{equation}
followed by $\cleanupop(\tilde{\mathbf{z}}_t)$. Ranking uses cosine similarity
to all entity embeddings.

\paragraph{Two-hop compositional prediction.}
For a chain $(r_1, r_2)$ with head $h$, we retrieve the intermediate entity:
\begin{equation}
  \hat{\mathbf{m}} = \cleanupop_{\mathrm{hard}}\!\bigl(\mem \unbind (\ev{h} \bind \rv{r_1})\bigr),
  \label{eq:hop1}
\end{equation}
then use $\hat{\mathbf{m}}$ as the head for the second hop:
\begin{equation}
  \hat{\mathbf{t}} = \cleanupop\!\bigl(\mem \unbind (\hat{\mathbf{m}} \bind \rv{r_2})\bigr).
  \label{eq:hop2}
\end{equation}
Both hops query the same memory $\mem$; no intermediate entity is ever observed.
Section~\ref{sec:results-mechanism} shows that hop~1 succeeds at high accuracy,
and that the bottleneck is in retrieving the second-hop fact from $\mem$,
regardless of whether $\hat{\mathbf{m}}$ is correct.

\subsection{Theoretical Analysis: Softmax Cleanup is Phase-Nonequivariant}
\label{sec:theory}

The mechanistic probes in Section~\ref{sec:results-mechanism} show that the
\emph{primary} bottleneck, retrieval capacity under superposition, is
measurable at a single hop, before any cleanup step acts. We additionally
identify a \emph{secondary}, FHRR-specific failure that compounds this primary
effect whenever hop-1 cleanup is imperfect: the softmax Hopfield cleanup
followed by phase re-projection does not commute with binding.

\begin{lemma}[Softmax cleanup is phase-nonequivariant]
\label{lem:softmax-noncommute}
Let $\mathbf{z} \in \Cb^D$ with $|z_k|=1$ for all $k$, and let
$\boldsymbol{\rho} \in \Cb^D$ with $|\rho_k|=1$ be a unit-phasor relation
embedding. Define the FHRR Hopfield cleanup $\cleanupop_{\mathrm{F}}$ as in
Eq.~\eqref{eq:cleanup-fhrr}, with similarities computed as
$\Real(\ev{i}^\top \bar{\mathbf{z}})$. Then in general,
\begin{equation}
  \cleanupop_{\mathrm{F}}(\mathbf{z} \odot \boldsymbol{\rho})
  \;\neq\;
  \cleanupop_{\mathrm{F}}(\mathbf{z}) \odot \boldsymbol{\rho}.
  \label{eq:noncommute}
\end{equation}
\end{lemma}

\begin{proof}
It suffices to exhibit a single component $k$ for which
Eq.~\eqref{eq:noncommute} fails. Consider $D=1$ with a two-entity
codebook $\{e^{i\phi_1}, e^{i\phi_2}\} \subset \Cb$, query $z = e^{i\theta}$,
and binding phasor $\rho = e^{i\psi}$.

The softmax weights for query $z$ are
$\alpha_j = \exp(\beta\cos(\phi_j - \theta)) / Z(\theta)$,
where $Z(\theta) = \sum_\ell \exp(\beta\cos(\phi_\ell - \theta))$.
The cleanup output is $c(z) = \exp(i\,\angle(\alpha_1 e^{i\phi_1} + \alpha_2 e^{i\phi_2}))$.

\textit{Left-hand side} of Eq.~\eqref{eq:noncommute}: query is 
$z \odot \rho = e^{i(\theta+\psi)}$, so weights become
$\tilde\alpha_j = \exp(\beta\cos(\phi_j - \theta - \psi)) / Z(\theta+\psi)$,
giving
\[
  \mathrm{LHS} = \exp\!\bigl(i\,\angle\bigl(\tilde\alpha_1 e^{i\phi_1} + \tilde\alpha_2 e^{i\phi_2}\bigr)\bigr).
\]

\textit{Right-hand side}: multiply cleanup output by $\rho$:
\[
  \mathrm{RHS} = \exp\!\bigl(i\bigl[\angle(\alpha_1 e^{i\phi_1} + \alpha_2 e^{i\phi_2}) + \psi\bigr]\bigr).
\]

For LHS $=$ RHS one would need
$\angle(\tilde\alpha_1 e^{i\phi_1} + \tilde\alpha_2 e^{i\phi_2})
= \angle(\alpha_1 e^{i\phi_1} + \alpha_2 e^{i\phi_2}) + \psi$.
But $\tilde\alpha_j \neq \alpha_j$ whenever $\psi \neq 0$ and the two entities
are not symmetrically placed around $\theta$, so the weighted sum of phasors
is not simply phase-shifted by $\psi$.

\textit{Concrete counterexample.}
Set $\phi_1 = 0$, $\phi_2 = 2\pi/3$, $\theta = 0.1$, $\psi = \pi/3$,
$\beta = 5$.
Then $\alpha_1 \approx 0.9991$, $\alpha_2 \approx 0.0009$,
so $c(z) \approx e^{i \cdot 0.0008}$ and
$\mathrm{RHS} \approx e^{i \cdot 1.048}$.
For LHS, the query shifts to $\theta + \psi \approx 1.147$, now closer to
$\phi_2 \approx 2.094$ than to $\phi_1 = 0$:
$\tilde\alpha_1 \approx 0.296$, $\tilde\alpha_2 \approx 0.704$,
so $\mathrm{LHS} \approx e^{i \cdot 1.663}$.
We have $\mathrm{LHS} \approx e^{i \cdot 1.663} \neq e^{i \cdot 1.048} = \mathrm{RHS}$,
a phase gap of $\approx 0.6$ rad. For $D > 1$, errors compound
independently per component across the second binding in Eq.~\eqref{eq:hop2}.
\end{proof}

\begin{remark}
\label{rem:hardcleanup}
A hard nearest-neighbour cleanup is equivariant if and only if the phase shift
$\psi$ maps each codebook Voronoi cell to another valid cell, which holds when
the codebook comes from a uniform phase lattice. Resonator
networks~\cite{frady2020resonator} provide an alternative that maintains phase
equivariance through iterative refinement. The primary bottleneck, however, lies
upstream of cleanup non-commutativity entirely. A phase-equivariant cleanup
addresses the secondary failure mode identified here but would not by itself
resolve the capacity-limited retrieval problem.
\end{remark}

\subsection{Training Objective}
\label{sec:method-training}

Both models minimise a cross-entropy loss over the atomic retrieval task plus a
contrastive invertibility regulariser:
$\mathcal{L} = \mathcal{L}_{\text{atom}} + \lambda \mathcal{L}_{\text{inv}}$,
$\lambda = 0.2$. $\mathcal{L}_{\text{inv}}$ encourages $f \bind r$ to unbind
back to both $f$ and $r$. The memory $\mem$ is Hebbian-initialised and refined
by gradient descent. We use Adam (lr $10^{-3}$, cosine-annealed to $10^{-5}$),
batch size $2048$, gradient clip $1.0$.

\begin{figure}[H]
\centering
\resizebox{\textwidth}{!}{%
\begin{tikzpicture}[
    font=\small,
    box/.style={rectangle, rounded corners=2pt, draw=black, thick,
                minimum width=1.4cm, minimum height=0.8cm, align=center},
    emb/.style={box, fill=blue!10},
    op/.style={circle, draw=black, thick, minimum size=0.65cm, fill=orange!25},
    mem/.style={box, fill=green!12, minimum width=1.1cm},
    clean/.style={box, fill=purple!12, minimum width=1.5cm},
    outbox/.style={box, fill=red!10},
    arr/.style={-{Stealth[length=2mm]}, thick},
    lbl/.style={font=\scriptsize\itshape, text=black!70}
]
\node[lbl] at (-1.4, 2.5) {\textbf{(a) Atomic}};
\node[emb]   (h)  at (0, 1.55) {$\ev{h}$};
\node[emb]   (r)  at (0, 0.45) {$\rv{r}$};
\node[op]    (b1) at (1.5, 1.0) {$\bind$};
\node[mem]   (M1) at (3.1, 1.85){$\mem$};
\node[op]    (u1) at (3.1, 1.0) {$\unbind$};
\node[clean] (c1) at (4.9, 1.0) {cleanup};
\node[outbox] (o1) at (6.7, 1.0) {$\hat t$};
\draw[arr] (h) -- (b1); \draw[arr] (r) -- (b1);
\draw[arr] (b1) -- (u1); \draw[arr] (M1) -- (u1);
\draw[arr] (u1) -- (c1) node[midway,above,lbl]{noisy};
\draw[arr] (c1) -- (o1);
\node[lbl] at (-1.25, -1.05) {\textbf{(b) Two-hop}};
\node[emb]   (h2) at (0, -1.85)   {$\ev{h}$};
\node[emb]   (ra) at (0, -2.95)   {$\rv{r_1}$};
\node[op]    (bb1)at (1.4, -2.4)  {$\bind$};
\node[op]    (uu1)at (2.9, -2.4)  {$\unbind$};
\node[mem]   (M2) at (2.9, -1.55) {$\mem$};
\node[clean] (cc1)at (4.5, -2.4)  {cleanup\\(hard)};
\node[emb]   (rb) at (4.5, -3.65) {$\rv{r_2}$};
\node[op]    (bb2)at (6.0, -2.85) {$\bind$};
\node[op]    (uu2)at (7.4, -2.85) {$\unbind$};
\node[mem]   (M3) at (7.4, -1.95) {$\mem$};
\node[clean] (cc2)at (9.0, -2.85) {cleanup};
\node[outbox] (o2) at (10.7, -2.85){$\hat t$};
\draw[arr] (h2) -- (bb1); \draw[arr] (ra) -- (bb1);
\draw[arr] (bb1) -- (uu1); \draw[arr] (M2) -- (uu1);
\draw[arr] (uu1) -- (cc1);
\draw[arr] (cc1) -- (bb2) node[pos=0.55,above,yshift=8pt,lbl]{$\hat m$};
\draw[arr] (rb) -- (bb2);
\draw[arr] (bb2) -- (uu2); \draw[arr] (M3) -- (uu2);
\draw[arr] (uu2) -- (cc2); \draw[arr] (cc2) -- (o2);
\node[draw=red, thick, dashed, rounded corners, fit=(uu2)(cc2)(o2),
      inner sep=4pt] {};
\node[red, font=\scriptsize\bfseries] at (9.3, -3.75) {capacity-limited retrieval};
\end{tikzpicture}%
}
\caption{Holographic memory architecture. \textbf{(a)} An atomic query binds the
head with the relation, unbinds from the superposed memory, and projects via the
Hopfield cleanup. \textbf{(b)} A two-hop query repeats bind-unbind-cleanup.
Hop~1 (left cleanup) succeeds with high fidelity; the dashed red box marks
where failure actually occurs: in retrieving the second-hop fact from the
superposed memory, not in the correctness of $\hat m$ or the bind-unbind
algebra.}
\label{fig:architecture}
\end{figure}

% ============================================================
\section{Experimental Setup}
\label{sec:experiments}
% ============================================================

\subsection{Dataset and Preprocessing}
We use FB15k-237~\cite{toutanova2015observed}: $14{,}541$ entities, $237$
relations, $272{,}115$/$17{,}535$/$20{,}466$ train/valid/test triples. The
zero-shot evaluation set follows Algorithm~\ref{alg:split}. The 10 selected
chains span 19 distinct relation types with mid-entity fan-out from $2.4$ to
$18.7$. Compositional failure is consistent across all chains.

\subsection{Implementation Details}
Both models use $D=1024$ real parameters (HRR: $D=1024$ real; FHRR: $D=512$
complex, i.e.\ $1024$ reals). Embeddings are initialised from
$\mathcal{N}(0, D^{-1}\mathbf{I})$. We use Adam ($\eta=10^{-3}$, cosine-annealed
to $10^{-5}$), batch size $2048$, gradient clipping $1.0$, $200$ epochs.
All experiments run over five random seeds $\{1,2,3,4,42\}$; we report
mean $\pm$ standard deviation. Default inverse temperatures: $\beta=8$ (HRR),
$\beta=12$ (FHRR).

\subsection{Phase Probe Definitions}
\label{sec:probes}
For complex vectors $\mathbf{u}, \mathbf{v} \in \Cb^D$ with $|u_k|=|v_k|=1$:
\begin{align}
  S_\phi(\mathbf{u}, \mathbf{v})
    &= \frac{1}{D}\Real\!\bigl(\mathbf{u}^\top \bar{\mathbf{v}}\bigr)
     = \frac{1}{D}\sum_{k=1}^D \cos(\angle u_k - \angle v_k),
  \label{eq:phasor-cos}\\
  \Delta\phi(\mathbf{u}, \mathbf{v})
    &= \frac{1}{D}\sum_{k=1}^D |\angle u_k - \angle v_k|_\pi,
  \label{eq:phase-err}
\end{align}
where $|\cdot|_\pi$ denotes the wrapped distance on $[-\pi,\pi]$.
For independent uniform random phasors,
$\mathbb{E}[S_\phi]=0$ and $\mathbb{E}[\Delta\phi]=\pi/2 \approx 1.571$\,rad.

\subsection{Mechanistic Probes}
\label{sec:probes-mechanism}

To localise the compositional failure, we introduce two further probes,
both computed by inference over stored checkpoints with no additional training.

\paragraph{Hop-1 mid-entity retrieval probe.}
For each unique $(h, r_1)$ pair in the zero-shot evaluation set, we treat
intermediate-entity prediction as a standard filtered atomic query: gold labels
are all $m$ with $(h, r_1, m)$ in the full graph (train $\cup$ valid $\cup$
test), and we report filtered MRR and Hits@$k$ using the standard protocol,
applied to $\mem \unbind (\ev{h} \bind \rv{r_1})$ before any second-hop
computation.

\paragraph{Composition conditional on mid correctness.}
For each zero-shot quadruple $(h, r_1, r_2, t)$, we compute a deterministic
intermediate prediction $\hat m = \argmaxop_i\,
\ev{i}^\top\bigl(\mem \unbind (\ev{h}\bind\rv{r_1})\bigr)$, re-embed
$\hat m$ as a clean codebook vector, and complete the second hop. We then
partition quadruples by whether $\hat m$ is a valid $r_1$-neighbour of $h$
and report composition accuracy on each partition separately.

\paragraph{Ground-truth hop-2 atomic-difficulty probe.}
For each zero-shot quadruple with at least one true chain-consistent intermediate
$m^\star$, meaning $(h,r_1,m^\star)$ and $(m^\star, r_2, t)$ both hold, we pose
$(m^\star, r_2)$ as a standalone atomic query using the model's own learned
embedding for $m^\star$, bypassing hop 1 entirely, and measure filtered top-1
accuracy. Comparing this to accuracy on a uniformly sampled atomic test query
isolates whether chain-relevant facts are intrinsically harder to retrieve from
$\mem$, independent of any error introduced by hop 1 or the composition
pipeline. We additionally stratify by relation fan-out (median split) to rule
out fan-out as the explanation.

% ============================================================
\section{Results}
\label{sec:results}
% ============================================================

\subsection{Atomic Retrieval}
\label{sec:results-atomic}

Table~\ref{tab:atomic} reports filtered single-hop performance over five seeds
on the full test set ($20{,}466$ queries). Real HRR reaches MRR
$0.358 \pm 0.002$ and FHRR $0.350 \pm 0.021$, both in the range of standard
baselines. The larger standard deviation for FHRR ($\pm 0.021$) likely reflects
greater sensitivity of complex-valued optimisation to random seeds; it does not
affect the qualitative picture. The Hopfield cleanup accounts for roughly half
of atomic performance (Table~\ref{tab:ablation}): removing it drops real-HRR
top-1 from $0.158$ to $0.081$.

\begin{table}[H]
\centering
\caption{Atomic (single-hop) link prediction on FB15k-237, filtered setting,
full test set. Our values are mean $\pm$ std over five seeds. Baseline values
are representative literature figures at their own dimensionalities.}
\label{tab:atomic}
\begin{tabular}{lccccc}
\toprule
Model & Top-1 & MRR & Hits@1 & Hits@3 & Hits@10 \\
\midrule
TransE$^\dagger$    & --   & 0.294 & --     & --     & 0.465 \\
DistMult$^\dagger$  & --   & 0.241 & --     & --     & 0.419 \\
ComplEx$^\dagger$   & --   & 0.247 & --     & --     & 0.428 \\
RotatE$^\dagger$    & --   & 0.338 & 0.241  & 0.375  & 0.533 \\
\midrule
Real HRR (ours)     & $0.158_{\pm 0.001}$ & $0.358_{\pm 0.002}$
 & $0.267_{\pm 0.002}$ & $0.392_{\pm 0.003}$ & $0.540_{\pm 0.003}$ \\
FHRR (ours)         & $0.126_{\pm 0.001}$ & $0.350_{\pm 0.021}$
 & $0.262_{\pm 0.017}$ & $0.390_{\pm 0.024}$ & $0.524_{\pm 0.028}$ \\
\bottomrule
\end{tabular}
\vspace{2pt}
{\footnotesize $^\dagger$ Literature values from \cite{bordes2013transe,
yang2014distmult,trouillon2016complex,sun2019rotate}; not re-run here.}
\end{table}

\begin{table}[H]
\centering
\caption{Core-component ablation (Real HRR, five seeds, full test set).}
\label{tab:ablation}
\begin{tabular}{lcc}
\toprule
Configuration & Atomic Top-1 & Zero-Shot Accuracy \\
\midrule
Full model            & $0.158 \pm 0.001$ & $1.7 \times 10^{-4}_{\pm 0.9 \times 10^{-4}}$ \\
$-$ Hopfield cleanup  & $0.081 \pm 0.010$ & $3.0 \times 10^{-4}_{\pm 1.0 \times 10^{-4}}$ \\
\bottomrule
\end{tabular}
\end{table}

\subsection{Zero-Shot Compositional Reasoning}
\label{sec:results-zeroshot}

Two-hop zero-shot accuracy is at or near chance for both variants: mean
accuracy is $1.7 \times 10^{-4}$ (real HRR) and $2.9 \times 10^{-5}$ (FHRR),
against chance $6.77 \times 10^{-5}$. We run a one-sided binomial test per
seed ($H_0$: accuracy $=$ chance). For FHRR, the null is not rejected at
$\alpha = 0.05$ in any seed (per-seed $p \in [0.34, 1.0]$). For real HRR,
the null is rejected in three of five seeds, but the absolute accuracy
($\sim 1.7 \times 10^{-4}$, a handful of the $69{,}855$ pairs) is orders of
magnitude below single-hop top-1, so the model is functionally at chance even
where the test formally rejects. Sweeping $\beta \in \{1,5,10,20,50\}$ leaves
accuracy flat near chance (Figure~\ref{fig:beta}): no temperature elicits
composition.

\begin{figure}[H]
\centering
\begin{tikzpicture}
\begin{axis}[
    width=0.85\textwidth, height=4.8cm,
    xlabel={Cleanup temperature $\beta$},
    ylabel={Zero-shot accuracy},
    xmode=log, log basis x=10,
    ymin=0, ymax=0.0009,
    xtick={1,5,10,20,50}, xticklabels={1,5,10,20,50},
    scaled y ticks=false, yticklabel style={/pgf/number format/sci},
    grid=both, grid style={gray!20},
    mark size=2.5pt,
]
\addplot[blue, thick, mark=*] coordinates {(1,0.0005)(5,0.0005)(10,0.0005)(20,0.0003)(50,0.0002)};
\addplot[red, dashed, thick] coordinates {(1,0.000068)(50,0.000068)};
\legend{Real HRR zero-shot, chance ($1/|\mathcal{E}|$)}
\end{axis}
\end{tikzpicture}
\caption{Cleanup temperature sweep (Real HRR, representative seed). Zero-shot
accuracy stays near chance across all $\beta$; no temperature elicits composition.}
\label{fig:beta}
\end{figure}
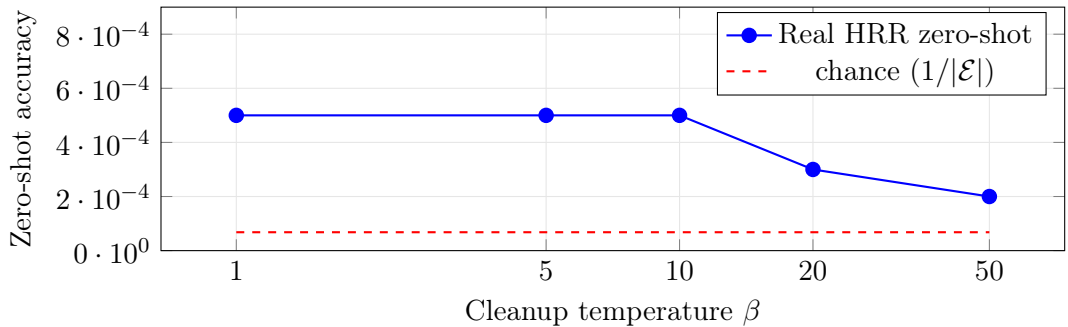

\subsection{Hop-1 Retrieval Succeeds: Localising the Failure to Hop 2}
\label{sec:results-hop1}

The most obvious explanation for compositional failure is that the model never
recovers a usable intermediate entity at hop 1, and that two-hop failure is
just hop-1 failure propagated forward. The hop-1 mid-entity retrieval probe
rules this out directly.

\begin{table}[H]
\centering
\caption{Hop-1 intermediate-entity retrieval, evaluated as a standard filtered
atomic query over the $2{,}020$ unique $(h,r_1)$ pairs in the zero-shot set.
Mean $\pm$ std over five seeds.}
\label{tab:hop1}
\begin{tabular}{lcccc}
\toprule
Model & MRR & Hits@1 & Hits@10 & Median rank \\
\midrule
Real HRR & $0.896_{\pm 0.002}$ & $0.845_{\pm 0.003}$ & $0.976_{\pm 0.005}$ & $1.0$ \\
FHRR     & $0.849_{\pm 0.034}$ & $0.777_{\pm 0.049}$ & $0.961_{\pm 0.010}$ & $1.0$ \\
\bottomrule
\end{tabular}
\end{table}

Hop-1 retrieval is excellent: median rank $1$, Hits@1 above $0.84$ for real
HRR and $0.78$ for FHRR, substantially better than the models' own atomic
test accuracy (Table~\ref{tab:atomic}, Hits@1 $\approx 0.27$/$0.26$). The
gap is expected: $(h, r_1)$ pairs in high-frequency chains tend to be
well-represented training facts, not a uniform draw from the test distribution.
Whatever causes two-hop failure, it is not an inability to identify the
intermediate entity.

We next ask whether composition succeeds \emph{conditional} on a correct
intermediate. Table~\ref{tab:conditional} shows it does not.

\begin{table}[H]
\centering
\caption{Composition accuracy conditioned on whether the deterministic
intermediate prediction $\hat m$ is a valid $r_1$-neighbour of $h$. Mean
$\pm$ std over five seeds; $n$ is the mean partition size out of $69{,}855$
quadruples. Chance is $6.77\times10^{-5}$.}
\label{tab:conditional}
\begin{tabular}{lccc}
\toprule
Model & Acc.\ $\mid \hat m$ valid & Acc.\ $\mid \hat m$ invalid & Frac.\ $\hat m$ valid \\
\midrule
Real HRR & $2.89_{\pm 1.28}\times10^{-4}$ ($n{\approx}15{,}946$) & $1.25_{\pm 0.99}\times10^{-4}$ ($n{\approx}53{,}909$) & $0.228_{\pm 0.051}$ \\
FHRR     & $7.05_{\pm 3.60}\times10^{-4}$ ($n{\approx}13{,}590$) & $0.60_{\pm 0.48}\times10^{-4}$ ($n{\approx}56{,}264$) & $0.195_{\pm 0.087}$ \\
\bottomrule
\end{tabular}
\end{table}

Composition accuracy given a valid intermediate is higher than given an invalid
one for both models. But both conditional accuracies remain within one order of
magnitude of chance, two to four orders of magnitude below atomic Hits@1.
At the per-seed level, real HRR produces as few as $1$ correct prediction out
of $15{,}815$ quadruples with a verified-correct intermediate (seed 4); FHRR
produces as few as $3$ out of $10{,}603$ (seed 42). A correct intermediate entity
is necessary but nowhere near sufficient for correct composition. The bottleneck
sits downstream of hop-1 retrieval, in the second bind-unbind-cleanup step itself.

\subsection{Hop-2 Atomic-Difficulty Probe: The Bottleneck Is Retrieval Capacity}
\label{sec:results-mechanism}

Section~\ref{sec:results-hop1} shows hop-1 succeeds and composition still fails
even with a verified-correct intermediate. Two explanations remain. One:
something specific to the composition pipeline, namely re-embedding $\hat m$,
re-binding with $\rv{r_2}$, re-unbinding from $\mem$, introduces a failure
absent from an ordinary atomic query. Two: the specific facts that
compositional chains pass through are harder for the memory to retrieve than
a typical fact, a difficulty that exists at a single hop and has nothing to do
with composition mechanics. We distinguish these using the ground-truth hop-2
atomic-difficulty probe: we take the true chain-consistent intermediate
$m^\star$ from the graph (bypassing hop 1 entirely) and pose $(m^\star, r_2)$
as a standalone atomic query, identical in form to those underlying
Table~\ref{tab:atomic}.

\begin{table}[H]
\centering
\caption{Ground-truth hop-2 facts evaluated as standalone atomic queries
(filtered top-1), compared to the model's standard atomic test accuracy on
the same checkpoints, over $200$ unique $(m^\star, r_2)$ queries. Mean $\pm$
std over five seeds.}
\label{tab:hop2atomic}
\begin{tabular}{lccc}
\toprule
Model & Standard atomic Top-1 & Hop-2 fact Top-1 & Ratio \\
\midrule
Real HRR & $0.267_{\pm 0.002}$ & $0.070_{\pm 0.016}$ & $0.26\times$ \\
FHRR     & $0.262_{\pm 0.017}$ & $0.126_{\pm 0.057}$ & $0.48\times$ \\
\bottomrule
\end{tabular}
\end{table}

The result is clear. Even with no hop-1 error and no composition pipeline at
all, accuracy on the exact facts chains depend on is degraded to a quarter
(real HRR) to a half (FHRR) of the models' accuracy on a typical atomic query.
Stratifying by relation fan-out (median split at $37.5$) shows this degradation
is uniform: real HRR scores $0.068$ top-1 on low-fan-out hop-2 facts and
$0.072$ on high-fan-out ones, which rules out fan-out as the explanation. The
failure is not concentrated in high-fan-out relations, and it is not introduced
by the composition pipeline. It is a property of which facts are hard to retrieve
from $\mem$, visible already at a single unbind operation.

\paragraph{Why these specific facts are hard.}
The facts selected by Algorithm~\ref{alg:split} are second legs of
high-support relation chains: $r_2$ relations and $m^\star$ entities that
participate in many chain instances, and therefore in many competing Hebbian
terms in $\mem$. This is consistent with the $O(\sqrt{K/D})$ cross-talk noise
scaling from Section~\ref{sec:method-algebra}: facts involving entities or
relations with higher effective participation in the superposition are recovered
with lower fidelity, even though the average fact (as sampled by the standard
test set) is recovered adequately. Composition inherits this weakness by
construction, since compositionally chained facts are precisely the
higher-degree, more contested entries in $\mem$.

\paragraph{FHRR modulus and phase probes.}
A common intuition for phasor memories under iterated binding is modulus
collapse: the softmax cleanup is a convex combination of unit phasors, so its
output can have $|z| \ll 1$, decaying over hops. This turns out not to apply
here. The mean per-component modulus of the cleaned representation is $|z| = 1.0000$
(std $0.0000$) at both hops, across all five seeds, because the FHRR cleanup
re-projects onto the unit torus by construction (Eq.~\ref{eq:cleanup-fhrr}).
Magnitude-based diagnostics are vacuous.

The relevant question is then whether phase information survives.
Figure~\ref{fig:coherence} shows phasor cosine similarity between the cleaned
two-hop representation and the true target versus a random entity. Similarity
to the true target is $-0.009 \pm 0.006$, statistically indistinguishable from
$0.005 \pm 0.003$ to a random entity: the cleaned representation retains no
directional information at the final hop. The mean absolute phase error at hop
2 is $1.59$ rad, essentially equal to the uninformative baseline
$\pi/2 \approx 1.571$ (Figure~\ref{fig:phaseerr}); circular variance is $0.99$
against a uniform value of $1$. Given the retrieval-capacity bottleneck
established in Sections~\ref{sec:results-hop1} and \ref{sec:results-mechanism},
this near-uniform phase distribution is consistent with the cleanup
non-commutativity of Lemma~\ref{lem:softmax-noncommute} acting on an already
degraded retrieval signal, though it is not uniquely caused by it.

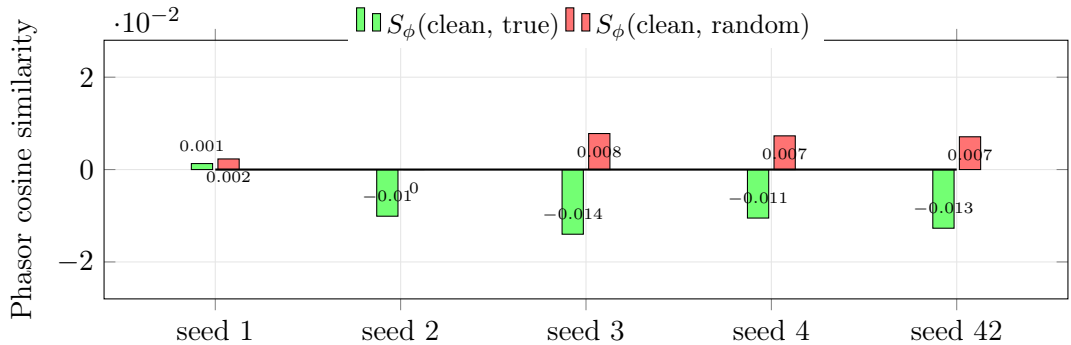
\begin{figure}[H]
\centering
\begin{tikzpicture}
\begin{axis}[
    width=0.9\textwidth, height=5cm,
    ybar, bar width=8pt,
    ylabel={Phasor cosine similarity},
    symbolic x coords={s1, s2, s3, s4, s42},
    xtick=data, xticklabels={seed 1, seed 2, seed 3, seed 4, seed 42},
    ymin=-0.028, ymax=0.028,
    legend style={at={(0.5,1.15)}, anchor=north, legend columns=2, draw=none, font=\small},
    enlarge x limits=0.15,
    grid=both, grid style={gray!20},
]
\addplot[fill=green!55, nodes near coords,
    every node near coord/.append style={anchor=south, font=\tiny, yshift=1pt,
    /pgf/number format/fixed, /pgf/number format/precision=3}]
    coordinates {(s1,0.0013)(s2,-0.0101)(s3,-0.0140)(s4,-0.0105)(s42,-0.0127)};
\addplot[fill=red!55, nodes near coords,
    every node near coord/.append style={anchor=north, font=\tiny, yshift=-1pt,
    /pgf/number format/fixed, /pgf/number format/precision=3}]
    coordinates {(s1,0.0023)(s2,0.0000)(s3,0.0078)(s4,0.0073)(s42,0.0071)};
\draw[black, thick] (axis cs:s1,0) -- (axis cs:s42,0);
\legend{{$S_\phi$(clean, true)}, {$S_\phi$(clean, random)}}
\end{axis}
\end{tikzpicture}
\caption{FHRR phase coherence per seed: phasor cosine of the cleaned two-hop
representation to the true target (green) vs.\ a random entity (red). Both hover
near zero.}
\label{fig:coherence}
\end{figure}

\begin{figure}[H]
\centering
\begin{tikzpicture}
\begin{axis}[
    width=0.85\textwidth, height=5cm,
    ybar, bar width=14pt,
    ylabel={Value},
    symbolic x coords={Hop 1, Hop 2},
    xtick=data,
    ymin=0, ymax=1.95,
    legend style={at={(0.5,1.12)}, anchor=north, legend columns=2, draw=none, font=\small},
    enlarge x limits=0.5,
    grid=both, grid style={gray!20},
]
\addplot[fill=blue!55] coordinates {(Hop 1,1.537)(Hop 2,1.586)};
\addplot[fill=orange!70] coordinates {(Hop 1,0.973)(Hop 2,0.987)};
\draw[red, thick, dashed] (axis cs:Hop 1,1.571) -- (axis cs:Hop 2,1.571)
   node[pos=0.5, above, font=\scriptsize, red] {$\pi/2$ (uninformative)};
\legend{Mean abs.\ phase error (rad), Circular variance}
\end{axis}
\end{tikzpicture}
\caption{Phase error propagation. Mean absolute phase error sits at the
uninformative $\pi/2$ baseline at both hops; circular variance approaches $1$.}
\label{fig:phaseerr}
\end{figure}
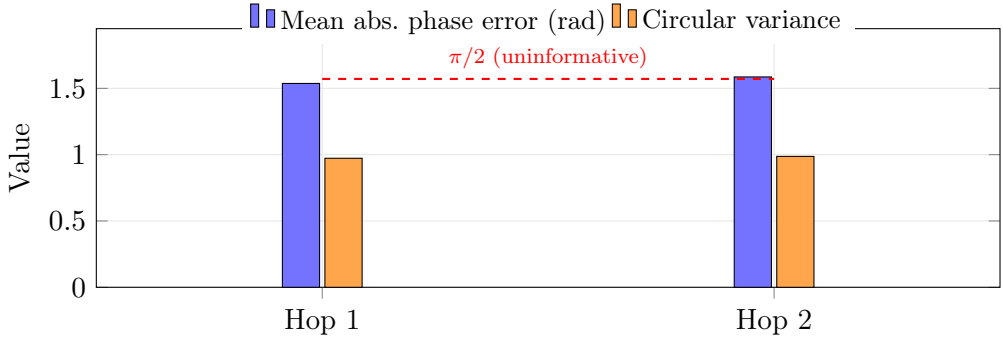

\paragraph{Renormalisation and hard cleanup.}
Explicitly re-normalising intermediate and final representations to unit modulus
gives zero-shot accuracy $< 1 \times 10^{-5}$ (zero correct in $69{,}855$
pairs). Replacing soft cleanup with hard argmax at both hops likewise gives
zero correct answers. Neither intervention is expected to help under the
retrieval-capacity explanation, since neither changes which facts are weakly
represented in $\mem$.

\paragraph{How atomic ranking survives the same capacity limit that defeats composition.}
FHRR achieves MRR $0.350$ overall while hop-2 fact top-1 accuracy is only
$0.126$. Here is the resolution: the atomic softmax margin (gold logit minus
the largest incorrect logit) is negative on average ($-1.82$ for FHRR, $-1.92$
for real HRR), and the gold entity is top-1 only $\approx 13$ to $16\%$ of the
time on the full test set, consistent with Table~\ref{tab:atomic}. Yet MRR
stays high because the gold entity is reliably ranked among the top few even
when it is not first. Composition has no such tolerance: a near-miss in the
final ranking is scored as a hard failure. The same capacity limitation that
costs ordinary atomic queries a few ranks costs compositional queries the
entire prediction.

\paragraph{Untrained control.}
A freshly initialised FHRR has atomic accuracy $\approx 0$ while the trained
model reaches $0.126 \pm 0.001$; training helps single-hop retrieval. Yet both
are equally at chance on two-hop composition. Training buys average atomic
competence but does not differentially improve retrieval of the specific,
higher-contention facts composition depends on.

% ============================================================
\section{Discussion}
\label{sec:discussion}
% ============================================================

The central finding is that zero-shot compositional failure in holographic
memory is a single-hop retrieval-capacity problem, not a multi-hop or
cleanup-algebra problem. Hop-1 retrieval succeeds with high fidelity; composition
fails even with a correct intermediate; and the ground-truth hop-2 fact, posed
as a standalone atomic query with no composition involved, is already degraded
to $0.26$ to $0.48\times$ the typical atomic accuracy, uniformly across relation
fan-out. The facts compositional reasoning depends on are, by the same construction
that makes them compositionally useful, exactly the higher-degree facts that
suffer the most cross-talk under superposition. The right question for future
work is therefore not whether a cleanup operator can be made compositionally
consistent, but whether a fixed-width superposed memory can retrieve
high-degree facts with high fidelity at all, before any query composition happens.

Lemma~\ref{lem:softmax-noncommute} identifies a genuine structural property of
the softmax-cleanup-plus-phase-reprojection pipeline in FHRR, and the phase
probes show it plausibly compounds the primary failure on the roughly $20\%$ of
chains where hop-1 retrieval is imperfect. But it is not the primary cause. The
dominant contributor is a retrieval-capacity effect already measurable at a
single hop, before cleanup non-commutativity has a chance to act.

Three intuitions deserve caution in light of these results. First, competitive
single-hop \emph{aggregate} performance does not imply that any individual
fact, especially the high-degree facts compositional reasoning depends on, is
retrieved reliably; aggregate metrics like MRR can mask substantial per-fact
variation. Second, the modulus collapse intuition does not apply once cleanup
re-normalises onto the unit torus; the relevant question then is whether
retrieval succeeds upstream of cleanup. Third, and most concretely, a correct
intermediate entity is not informative about whether composition will succeed.
The conditional accuracies in Table~\ref{tab:conditional} differ by less than a
factor of three between the valid- and invalid-intermediate partitions, a much
smaller gap than one might expect if intermediate correctness were the dominant
factor.

% ============================================================
\section{Limitations}
\label{sec:limitations}
% ============================================================

This study evaluates a specific pipeline on a single benchmark. We test only
two-hop chains; the retrieval-capacity effect identified here should compound
with chain length, since each additional hop queries another potentially
high-contention fact. Results on sparser graphs (e.g.\
WN18RR~\cite{dettmers2018wn18rr}) or under explicit multi-hop protocols (e.g.\
NELL-995~\cite{xiong2017deeppath}) may differ quantitatively; we predict the
capacity effect to be less severe there, given fewer superposed facts per
dimension. We do not empirically test a phase-equivariant cleanup or a memory
architecture with explicitly allocated capacity, such as per-relation
sub-memories, that might reduce the cross-talk effect. The hop-2
atomic-difficulty probe uses $200$ unique $(m^\star, r_2)$ queries, fewer than
the $20{,}466$-query standard test set; the effect size is large and consistent
across all five seeds, but a larger replication would sharpen the estimate. The
training objective targets single-hop retrieval; an objective that explicitly
up-weights high-degree or chain-relevant facts might partially compensate,
though we have not tested this.

% ============================================================
\section{Conclusion}
\label{sec:conclusion}
% ============================================================

Holographic reduced representations paired with modern Hopfield cleanup are
competitive single-hop retrievers on FB15k-237 but fail completely at zero-shot
two-hop composition. Two targeted mechanistic probes localise the failure: the
intermediate entity is retrieved with high fidelity at hop 1, yet composition
fails even when that intermediate is verified correct, and the ground-truth
second-hop fact, evaluated as a standalone atomic query with no composition
involved, is retrieved at only $0.26$ to $0.48\times$ the typical atomic
accuracy, uniformly across relation fan-out. The bottleneck is retrieval
capacity under superposition for the specific, higher-contention facts
compositional chains depend on, not the bind-unbind algebra or the cleanup
step. We additionally prove (Lemma~\ref{lem:softmax-noncommute}) that FHRR's
softmax cleanup is not phase-equivariant, a real but secondary effect that
compounds the primary failure only when hop-1 retrieval itself errs. Fixing
zero-shot composition in holographic memory requires addressing retrieval
fidelity for high-degree facts under superposition, through explicit capacity
allocation or chain-aware training, rather than cleanup-operator redesign alone.

% ============================================================
\section*{Reproducibility Statement}
% ============================================================
All results derive from five fixed-seed checkpoints under the leakage-controlled
protocol of Algorithm~\ref{alg:split}. Atomic metrics use the filtered
convention on the full $20{,}466$-query test set. The hop-1 retrieval probe,
the composition-conditional probe, and the hop-2 atomic-difficulty probe
(Section~\ref{sec:probes-mechanism}) are all inference-only over the same five
checkpoints, with no additional training. Phase probes, modulus probes, softmax
margins, temperature sweeps, binomial tests, and untrained controls are likewise
computed by inference over stored checkpoints. Code, configurations, and all
probe scripts are released at
\url{https://github.com/iamhero2709/holographic-memory}.

% ============================================================

\end{document}